\documentclass{article} 
\usepackage{nips13submit_e,times}
\usepackage{hyperref}
\usepackage{paralist}
\usepackage{amsmath}
\usepackage{amssymb}
\usepackage{wasysym}
\usepackage{algorithm}
\usepackage{algorithmic}
\usepackage{url}

\usepackage{fullpage} 

\newcommand{\ActionSet}{{\mathcal{A}}}

\newcommand{\InputSet}{{\mathcal{X}}}

\newcommand{\E}{{\operatorname{\mathbf{E}}}}  
\newcommand{\kld}[2]{{\operatorname{d_{KL}}\left(#1,#2\right)}}  

\newcommand{\vecb}[1]{{\mathbf{#1}}}
\newcommand{\abs}[1]{\left\lvert#1\right\rvert}  
\newcommand{\norm}[1]{\left\lVert#1\right\rVert}  
\newcommand{\Rset}{{\mathbb{R}}}  

\newcommand{\setcard}[1]{\left\lvert#1\right\rvert}  

\newcommand{\indfunc}[1]{{\mathbb{I}(#1)}}  

\newcommand{\defeq}{{\stackrel{\mathrm{def}}{=}}}

\newcommand{\assign}{{\leftarrow}}

\newenvironment{proof}[1][Proof]{\begin{trivlist}\item[\hskip \labelsep {\bfseries #1}]}{\end{trivlist}}

\newtheorem{lemma}{Lemma}
\newtheorem{theorem}{Theorem}

\newcommand{\eqnref}[1]{Equation~(\ref{#1})}
\newcommand{\lemref}[1]{Lemma~\ref{#1}}

\newcommand{\thmref}[1]{Theorem~\ref{#1}}

\newcommand{\condref}[1]{Condition~\ref{#1}}
\newcommand{\corref}[1]{Corollary~\ref{#1}}
\newcommand{\secref}[1]{Section~\ref{#1}}
\newcommand{\algref}[1]{Algorithm~\ref{#1}}

\newcommand{\eg}{\textit{e.g.}}

\newtheorem{corollary}{Corollary}

\renewenvironment{proof}{Proof.}{{$\qquad\Box$}}
\newcommand{\citet}[1]{\cite{#1}}
\newcommand{\citep}[1]{\cite{#1}}

\newcommand{\ts}{Thompson Sampling}
\newcommand{\gts}{Generalized Thompson Sampling}
\newcommand{\expert}{\mathcal{E}}
\newcommand{\cref}[1]{C\ref{#1}}
\renewcommand{\condref}[1]{Condition~C\ref{#1}}

\newcommand{\lihong}[1]{\textbf{[[LL: #1]]}}
\renewcommand{\lihong}[1]{}

\title{Generalized Thompson Sampling for \\ Contextual Bandits}

\author{
Lihong Li \\
Microsoft Research \\
Redmond, WA 98052 \\
\texttt{lihongli@microsoft.com} \\
}

\nipsfinalcopy 

\begin{document}

\maketitle

\begin{abstract}
\ts, one of the oldest heuristics for solving multi-armed bandits, has recently been shown to demonstrate state-of-the-art performance.  The empirical success has led to great interests in theoretical understanding of this heuristic.  In this paper, we approach this problem in a way very different from existing efforts.  In particular, motivated by the connection between \ts\ and exponentiated updates, we propose a new family of algorithms called \gts\ in the expert-learning framework, which includes \ts\ as a special case.  Similar to most expert-learning algorithms, \gts\ uses a loss function to adjust the experts' weights.  General regret bounds are derived, which are also instantiated to two important loss functions: square loss and logarithmic loss.  In contrast to existing bounds, our results apply to quite general contextual bandits.  More importantly, they quantify the effect of the ``prior'' distribution on the regret bounds.
\end{abstract}


\section{Introduction}

\ts~\cite{Thompson33Likelihood}, one of the oldest heuristics for solving stochastic multi-armed bandits, embodies the principle of \emph{probability matching}.  Given a prior distribution over the underlying, unknown reward generating process as well as past observations of rewards, one can maintain a posterior distribution of which arm is optimal.  \ts\ then selects arms randomly according to the current posterior distribution.

While having being unpopular for decades, this algorithm was recently shown to be state-of-the-art in empirical studies, and has found success in important applications like news recommendation and online advertising~\citep{Scott10Modern,Graepel10Web,Chapelle12Empirical,May12Optimistic}.  In addition, it has other advantages such as robustness to observation delay~\cite{Chapelle12Empirical} and simplicity in implementation, compared to the dominant strategies based on upper confidence bounds (UCB).

Despite the empirical success, theoretical understanding of finite-time performance of \ts\ has been limited until very recently.  The first such result is provided by \citet{Agrawal12Analysis} for \emph{non-contextual} $K$-armed bandits, who prove a nontrivial problem-dependent regret bound when the prior of an arm's expected reward is a Beta distribution.  Later on, improved bounds are found for the same setting~\citep{Kaufmann12Thompson,Agrawal13Further}, which match the asymptotic regret lower bound~\citet{Lai85Asymptotically}.

For contextual bandits~\citep{Langford08Epoch}, only two pieces of work are available, to the best of our knowledge.  \citet{Agrawal13Thompson} analyze linear bandits, where a Gaussian prior is used on the weight vector space, and a Gaussian likelihood function is assumed for the reward function.  The authors are able to show the regret grows on the order of $d\sqrt{T}$, which is only a $\sqrt{d}$ factor away from a known matching lower bound~\citep{Chu11Contextual}.  In contrast, \cite{Russo13Learning} establish an interesting connection between UCB-style analysis and the \emph{Bayes risk} of \ts, based on the probability-matching property.  This observation allows the authors to obtain Bayes risk bound based on a novel metric, known as \emph{margin dimension}, of an arbitrary function class that essentially measures how fast upper confidence bounds decay.

All the existing work above relies critically either on advanced properties of the assumed prior distribution (such as in the case of Beta distributions), or on the assumption that the prior is correct (in the analysis of Bayes risk of \cite{Russo13Learning}).  Such analysis, although very interesting and important for better understanding \ts, seems hard to be generalized to general (possibly nonlinear) contextual bandits.  Furthermore, none of the existing theory is able to quantify the role of prior plays in controlling the regret, although in practice better domain knowledge is often available to construct good priors that should ``accelerate'' learning.

This paper attempts to address the limitations of prior work, from a very different angle.  Based on a connection between \ts\ and exponentiated update rules, we propose a family of contextual-bandit algorithms called \gts\ in the expert-learning framework~\citep{Cesabianchi06Prediction}, where each expert corresponds to a contextual policy for arm selection.  Similar to \ts, \gts\ is a randomized strategy, following an expert's policy more often if the expert is more likely to be optimal.  Different from \ts, it uses a loss function to update the experts' weights; \ts\ is a special of \gts\ when the logarithmic loss is used.\footnote{It should be emphasized that, in this paper, we use the loss function to measure how well an expert \emph{predicts} the average reward, given the context and the selected arm.  In general, the loss function and the reward may be completely unrelated.  Details are given later.}

Regret bounds are then derived under certain conditions.  The proof relies critically on a novel application of a ``self-boundedness'' property of loss functions in competitive analysis.  The results are instantiated to the square and logarithmic losses, two important loss functions.  Not only do these bounds apply to quite general sets of experts, but they also quantify the impact of the prior distribution on regret.  These benefits come at a cost of a worse dependence on the number of steps.  However, we believe it is possible to close the gap with a more involved analysis, and the connection between (Generalized) \ts\ to expert-learning will likely lead to further interesting insights and algorithms in future work.

\section{Preliminaries}

Contextual bandits can be formulated as the following game between the learner and a stochastic environment.  Let $\InputSet$ and $\ActionSet$ be the sets of context and arms, and let $K=\setcard{\ActionSet}$.  At step $t=1,2,\ldots,T$:
\begin{compactitem}
\item{Learner observes the context $x_t\in\InputSet$, where $x_t$ can be chosen by an adversary.}
\item{Learner selects arm $a_t\in\ActionSet$, and receives reward $r_t\in\{0,1\}$, with expectation $\mu(x_t,a_t)$.}
\end{compactitem}
Note that the setup above allows the contexts to be chosen by an adversary, which is a more general setting than typical contextual bandits~\cite{Langford08Epoch}.  The reader may notice we require the reward to be binary, instead of being in $[0,1]$.  This choice will make our exposition simpler, without sacrificing loss of generality.  Indeed, as also suggested by \cite{Agrawal12Analysis}, if reward $r\in(0,1)$ is received, one can convert it into a binary pseudo-reward $\tilde{r}\in\{0,1\}$ as follows: let $\tilde{r}$ be $1$ with probability $r$, and $0$ otherwise.  Clearly, the bandit process remains essentially the same, with the same optimal expert and regrets.

Motivated by prior work on \ts\ with parametric function classes~\cite{Chapelle12Empirical}, we allow the learner to have access to a set of experts, $\expert=\{\expert_1,\ldots,\expert_N\}$, each one of them makes predicts about the average reward $\mu(x,a)$.  Let $f_i$ be the associated prediction function of expert $\expert_i$.  Its arm-selection policy in context $x$ is simply the greedy policy with respect to the reward predictions: $\expert_i(x)=\max_{a\in\ActionSet}f_i(x,a)$.  This setting can naturally be used to capture the use of parametric function classes: for example, when generalized linear models are used to predict $\mu(x,a)$~\cite{Graepel10Web,Chapelle12Empirical}, each weight vector is an expert.  The only difference is that our framework works with a discrete set of experts.  Using a covering device, however, it is possible to approximate a continuous function class by a finite set of cardinality $N$, where $N$ is the covering number.

We define the $T$-step average regret of the learner by
\begin{equation}
R(T) = \max_{1 \le i \le N} \sum_{t=1}^T \mu(x_t,\expert_i(x_t)) - \E\left[\sum_{t=1}^T \mu(x_t,a_t)\right] , \label{eqn:regret}
\end{equation}
where the expectation refers to the possible randomization of the learner in selecting $a_t$.  As in all existing analysis for \ts, we make the realization assumption that one of the experts, $\expert^*\in\expert$, correctly predicts the average reward.  Without loss of generality, let $\expert_1$ be this expert; in other words, $f_1(x,a)\equiv\mu(x,a)$.  Clearly, $\expert_1$ is the reward-maximizing expert, so $R(T)=\sum_t \mu(x_t,\expert_1(x_t)) - \E\left[\sum_t \mu(x_t,a_t)\right]$.
\lihong{Extensions without the realization assumption is discussed briefly in \secref{sec:discussions}.}

With the notation above, \ts\ can be described as follows.  It requires as input a ``prior'' distribution $\vecb{p}=(p_1,\ldots,p_N)\in\Rset_+^N$ over the experts, where $\norm{\vecb{p}}_1=1$.  Intuitively, $p_i$ may be interpreted as the prior probability that $\expert_i$ is the reward-maximizing expert.  The algorithm starts with the first ``posterior'' distribution $\vecb{w}_1=(w_{1,1},\ldots,w_{N,1})$ where $w_{i,1}=p_i$.  At step $t$, the algorithm samples an expert based on the posterior distribution $\vecb{w}_t$ and follows that expert's policy to choose action.  Upon receiving the reward, the weights are updated by $w_{i,t+1} \propto w_{i,t} \exp\left(-\ell(f_i(x_t,a_t),r_t)\right)$, where $\ell(f,r)$ is the negative log-likelihood.

Finally, one can assume the optimal expert, $\expert^*$, is drawn from an \emph{unknown} prior distribution, $\vecb{p}^*=(p^*_1,\ldots,p^*_N)$.  The expected $T$-step \emph{Bayes regret} can then be defined: $R(T,\vecb{p}^*) \defeq \E_{\expert^*\sim\vecb{p}^*}\left[R(T)\right]$.  It should be noted that the Bayes risk considered by other authors~\cite{Russo13Learning} is just $R(T,\vecb{p})$, where $\vecb{p}$ is the prior used by \ts.  In general, the true prior $\vecb{p}^*$ is unknown, so $\vecb{p}\ne\vecb{p}^*$.  We believe the Bayes risk defined with respect to $\vecb{p}^*$ is more reasonable in light of the almost inevitable misspecificatin of priors in practice.

\section{Generalized Thompson Sampling}

An observation with \ts\ from the previous section is that its Bayes update rule can be viewed as an exponentiated update with the logarithmic loss (see also \cite{Cesabianchi06Prediction}).  After receiving a reward, each expert is penalized for the mismatch in its prediction ($f_i$) and the observed reward, and the penalty happens to be the logarithmic loss in \ts.  Therefore, in principle, one can use other loss function to get a more general family of algorithms.  In fact, \emph{none} of the existing regret analyses~\cite{Agrawal12Analysis,Agrawal13Further,Agrawal13Thompson,Kaufmann12Thompson} relies on the interpretation that $\vecb{w}_t$ are meant to be Bayesian posteriors, and yet manages to show strong regret bound for \ts.\footnote{The analysis of \cite{Russo13Learning} is different since the metric (Bayes risk) is defined with respect to the prior.}  The above observations suggest the promising performance of \ts\ is not due to its Bayesian nature, and also motivates us to develop a more general family of algorithms known as \emph{\gts}.

We denote by $\ell(\hat{r},r)$ the loss incurred by reward prediction $\hat{r}$ when the observed reward is $r$.  \gts\ performs exponentiated updates to adjust experts' weights, and follows a randomly selected expert when making decisions, similar to \ts.  In addition, the algorithm also allows mixing of the exponentially weighted distribution and a uniform distribution controlled by $\gamma$.  The pseudocode is given in \algref{alg:gts}.


\begin{algorithm}
	\caption{Generalized Thompson Sampling}
	\label{alg:gts}
	\begin{algorithmic}
		\STATE \textbf{Input:} $\eta>0$, $\gamma>0$, $\{\expert_1,\ldots,\expert_N\}$, and prior $\vecb{p}$
		\STATE Initialize posterior: $\vecb{w}_1\assign\vecb{p}$; $W_1\assign\norm{\vecb{w}_1}_1=1$
		\FOR{$t=1,\dots,T$}
		\STATE Receive context $x_t\in\InputSet$
		\STATE Select arm $a_t$ according to the mixture probabilities: for each $a$
		\[
		\Pr(a) = (1-\gamma)\sum_{i=1}^N\frac{w_{i,t} \indfunc{\expert_i(x_t)=a}}{W_t} + \frac{\gamma}{K}
		\]
		\STATE Observe reward $r_t$, and updates weights:
		\[
		\forall i: w_{i,t+1} \assign w_{i,t}\cdot\exp\left(-\eta\cdot\ell\left(f_i(x_t,a_t),r_t\right)\right); \qquad W_{t+1}\assign\norm{\vecb{w}_{t+1}}_1=\sum_iw_{i,t+1}
		\]
		\ENDFOR
	\end{algorithmic}
\end{algorithm}

Clearly, \gts\ includes \ts\ as a special case, by setting $\eta=1$, $\gamma=0$, and $\ell$ to be the logarithmic loss: $\ell(\hat{r},r) = \indfunc{r=1}\ln1/\hat{r}+\indfunc{r=0}\ln1/(1-\hat{r})$.  Another loss function considered in this paper is the square loss: $\ell(\hat{r},r) = (\hat{r}-r)^2$.

\section{Analysis} \label{sec:analysis}

For convenience, the analysis here uses the following shorthand notation:
\begin{compactitem}
\item{The history of the learner up to step $t$ is $F_t\defeq(x_1,a_1,r_1,\ldots,x_{t-1},a_{t-1},r_{t-1},x_t)$.}
\item{The \emph{immediate regret} of expert $\expert_i$ in context $x$ is $\Delta_i(x)\defeq\mu(x,\expert_1(x)) - \mu(x,\expert_i(x))$.}
\item{The \emph{normalized weight} at step $t$ is $\bar{\vecb{w}}_t\defeq\vecb{w}_t/W_t$.}
\item{The \emph{shifted loss} incurred by expert $\expert_i$ in triple $(x,a,r)$ is denoted by $\hat{l}_i(r|x,a)\defeq\ell(f_i(x,a),r)-\ell(f_1(x,a),r)$.  In particular, define $\hat{l}_{i,t}\defeq\hat{l}_i(r_t|x_t,a_t)$.  In other words, $\hat{l}_i$ is the loss relative to the best expert ($\expert_1$), and \emph{can be negative}.}
\item{The \emph{average shifted loss} at step $t$ is $\bar{l}_t\defeq\E_{r_t,a_t|F_t}\left[\sum_i\bar{w}_{i,t}\hat{l}_i(r_t|x_t,a_t)\right]$.}
\end{compactitem}

\subsection{Main Theorem} \label{sec:main}

Clearly, conditions are needed to relate the loss function to the regret.  Our results need the following assumptions:
\begin{compactenum}[(C1)]
\item{(Consistency) For all $(x,a)\in\InputSet\times\ActionSet$, $\E_{r|x,a}\left[\hat{l}_i(r|x,a)\right] \ge 0$.} \label{cond:consistency}
\item{(Informativeness) There exists a constant $\kappa_1\in\Rset_+$ such that $\Delta_i(x_t) \le \kappa_1\sqrt{\bar{l}_t}$.} \label{cond:informativeness}
\item{(Boundedness) The shifted loss $\hat{l}_i$ assumes values in $[-1,1]$.} \label{cond:boundedness}
\item{(Self-boundedness) There exists a constant $\kappa_2\in\Rset_+$ such that, for all $(x,a)\in\InputSet\times\ActionSet$, $\E_{r|x,a}\left[\hat{l}_i(r|x,a)^2\right] \le \kappa_2 \E_{r|x,a}\left[\hat{l}_i(r|x,a)\right]$; namely, the second moment is bounded, up to a constant, by the first moment of the shifted loss.} \label{cond:selfboundedness}
\end{compactenum}

\begin{theorem} \label{thm:main}
Under Conditions~\cref{cond:consistency} and \cref{cond:informativeness}, the expected $T$-step regret of Generalized Thompson Sampling is
\[
R(T) \le \kappa_1\sqrt{T\cdot\E\left[\sum_{t=1}^T\bar{l}_t\right]} + \gamma T .
\]
\end{theorem}

\begin{proof}
The expected $T$-step regret may be rewritten more explicitly, and then bounded, as follows:
\begin{eqnarray*}
	R(T)
&\le& \E\left[\sum_{t=1}^T\left((1-\gamma)\sum_{i=1}^N\bar{w}_{i,t}\Delta_i(x_t)_+\gamma\right)\right] \\
&=& (1-\gamma)\E\left[\sum_t\sum_i\bar{w}_{i,t}\Delta_i(x_t)\right] + \gamma T \\
&\le& (1-\gamma)\E\left[\sum_t\sum_i\bar{w}_{i,t}\kappa_1\sqrt{\E_{r_t,a_t|F_t}\left[\hat{l}_i(r_t|x_t,a_t)\right]}\right] + \gamma T \\
&=& \kappa_1(1-\gamma)\E\left[\sum_t\sum_i\bar{w}_{i,t}\sqrt{\E_{r_t,a_t|F_t}\left[\hat{l}_i(r_t|x_t,a_t)\right]}\right] + \gamma T \\
&\le& \kappa_1(1-\gamma)\E\left[\sum_t\sqrt{\sum_i\bar{w}_{i,t}\E_{r_t,a_t|F_t}\left[\hat{l}_i(r_t|x_t,a_t)\right]}\right] + \gamma T \\
&\le& \kappa_1(1-\gamma)\sqrt{T\cdot\E\left[\sum_t\bar{l}_t\right]} + \gamma T .
\end{eqnarray*}
\end{proof}

Now the question becomes one of bounding the expected total shifted loss, $\E\left[\sum_t\bar{l}_t\right]$.  This problem is tackled by the following key lemma, which makes use the self-boundedness property of the loss function.  The lemma may be of interest on its own.  Similar properties were used in \cite{Agarwal12Contextual} in a very different way.

\begin{lemma} \label{lem:lbar}
Under Conditions~\cref{cond:boundedness} and \cref{cond:selfboundedness}, with $\eta$ chosen 
to be $(2(e-2)\kappa_2)^{-1}$,
the expected total shifted loss of \gts\ is bounded by a constant independent of $T$:
\[
\sum_{t=1}^T\bar{l}_t \le 4(e-2)\kappa_2\ln\frac{1}{p_1} .
\]
\end{lemma}

\begin{proof}
First, observe that if the shifted loss $\hat{l}_{i,t}$ is used in \gts\ to replace the loss $\ell(f_i(x_t,a_t),r_t)$, the algorithm behaves identically.  The rest of the proof uses this fact, pretending \gts\ uses $\hat{l}_{i,t}$ for weight updates.

For any step $t$, the weight sum changes according to 
\begin{eqnarray*}
\lefteqn{\ln\frac{W_{t+1}}{W_t} = \ln\left(\sum_i \bar{w}_{i,t} \exp\left(-\eta\hat{l}_{i,t}\right)\right)} \\
  &\le& \ln\left(\sum_i\bar{w}_{i,t}\left(1-\eta\hat{l}_{i,t}+(e-2)\eta^2\hat{l}_{i,t}^2\right)\right) \\
  &=& \ln\left(1-\eta\sum_i\bar{w}_{i,t}\hat{l}_{i,t}+(e-2)\eta^2\sum_i\bar{w}_{i,t}\hat{l}_{i,t}^2\right) \\
  &\le& -\eta\sum_i\bar{w}_{i,t}\hat{l}_{i,t}+(e-2)\eta^2\sum_i\bar{w}_{i,t}\hat{l}_{i,t}^2 ,
\end{eqnarray*}
where the first inequality is due to \condref{cond:boundedness} and the inequality $e^x\le1-x+(e-2)x^2$ for $x\in[-1,1]$; the second inequality is due to the inequality $\ln(1-x)<x$ for $x<1$.

Conditioned on the observed context and selected arm at step $t$, we take expectation of the above expressions, with respect to the randomization in observed reward, leading to
\begin{eqnarray*}
  \E_{r|F_t,a_t}\left[\ln\frac{W_{t+1}}{W_t}\right] \le -\eta\sum_i\bar{w}_{i,t}\E_{r|F_t,a_t}\left[\hat{l}_{i,t}\right]+(e-2)\eta^2\sum_i\bar{w}_{i,t}\E_{r|F_t,a_t}\left[\hat{l}_{i,t}^2\right] .
\end{eqnarray*}
\condref{cond:selfboundedness} then implies
\[
\E_{r|F_t,a_t}\left[\ln\frac{W_{t+1}}{W_t}\right] \le -\eta\left(1-(e-2)\eta\kappa_2\right)\sum_i\bar{w}_{i,t}\E_{r|F_t,a_t}\left[\hat{l}_{i,t}\right] .
\]
Setting $\eta=\left(2(e-2)\kappa_2\right)^{-1}$ gives
\[
\E_{r|F_t,a_t}\left[\ln\frac{W_{t+1}}{W_t}\right] \le -\frac{1}{4(e-2)\kappa_2}\sum_i\bar{w}_{i,t}\E_{r|F_t,a_t}\left[\hat{l}_{i,t}\right] .
\]
The above inequality holds for any $a_t$, so also holds in expectation if $a_t$ is randomized:
\[
\E_{a_t,r_t|F_t}\left[\ln\frac{W_{t+1}}{W_t}\right] \le -\frac{1}{4(e-2)\kappa_2}\sum_i\bar{w}_{i,t}\E_{a_t,r_t|F_t}\left[\hat{l}_{i,t}\right] .
\]
Finally, summing the left-hand side over $t=1,2,\ldots,T$ gives
\[
\E\left[\frac{W_{T+1}}{W_1}\right] \le -\frac{1}{4(e-2)\kappa_2}\sum_t\sum_i\bar{w}_{i,t}\E_{a_t,r_t|F_t}\left[\hat{l}_{i,t}\right] = -\frac{1}{4(e-2)\kappa_2}\E\left[\sum_t\bar{l}_t\right] ,
\]
which implies
\[
\E\left[\sum_t\bar{l}_t\right] \le 4(e-2)\kappa_2\E\left[\frac{W_1}{W_{T+1}}\right] \le 4(e-2)\kappa_2\ln\frac{1}{w_{1,1}} =  4(e-2)\kappa_2\ln\frac{1}{p_1}.
\]
The last inequality above follows from the observation that $\hat{l}_{1,t}\equiv0$, and that $w_{1,t}\equiv w_{1,1}$.
\end{proof}

The following corollary follows directly from \thmref{thm:main} and \lemref{lem:lbar}:
\begin{corollary} \label{cor:main}
Under the conditions of \thmref{thm:main} and \lemref{lem:lbar}, the expected $T$-step regret of \gts is at most
\[
\sqrt{4\kappa_2(e-2)}\kappa_1(1-\gamma)\sqrt{T\cdot\ln\frac{1}{p_1}} + \gamma T .
\]
\end{corollary}

The next corollary considers the Bayes regret, $R(T,\vecb{p}^*)$, with an unknown, true prior $\vecb{p}^*$:
\begin{corollary} \label{cor:average}
If the optimal expert is sampled from distribution $\vecb{p}^*$, the Bayes regret is at most
\[
R(T,\vecb{p}^*) \le \sqrt{4\kappa_2(e-2)}\kappa_1(1-\gamma)\sqrt{T\left(H(\vecb{p}^*)+\kld{\vecb{p}^*}{\vecb{p}}\right)} + \gamma T ,
\]
where $H(\vecb{p}^*)$ and $\kld{\vecb{p}^*}{\vecb{p}}$ are the standard entropy and KL-divergence.
\end{corollary}
\begin{proof}
We have
\begin{eqnarray*}
R(T,\vecb{p}^*) &\le& \sum_{i=1}^N p^*_i \left(\sqrt{4\kappa_2(e-2)}\kappa_1(1-\gamma)\sqrt{T\cdot\ln\frac{1}{w_{i,1}}} + \gamma T\right) \\
&=& \sqrt{4\kappa_2(e-2)}\kappa_1(1-\gamma)\sqrt{T} \sum_i\left(w^*_i\sqrt{\ln\frac{1}{w_i}}\right) + \gamma T \\
&\le& \sqrt{4\kappa_2(e-2)}\kappa_1(1-\gamma)\sqrt{T} \sqrt{\sum_i\left(w^*_i\ln\frac{1}{w_i}\right)} + \gamma T \\
&=& \sqrt{4\kappa_2(e-2)}\kappa_1(1-\gamma)\sqrt{T\left(H(\vecb{p}^*)+\kld{\vecb{p}^*}{\vecb{p}}\right)} + \gamma T ,
\end{eqnarray*}
where the inequalities are due to \corref{cor:main} and Jensen's inequality, respectively.
\end{proof}

\subsection{Square Loss} \label{sec:squareloss}

We start with the simpler case of square loss.  It clearly satisfies \condref{cond:boundedness}.  \condref{cond:consistency} holds because of the following well-known fact:
\[
	\E_{r|x,a}\left[\hat{l}_i(r|x,a)\right]
= \E_{r|x,a}\left[(f_i(x,a)-r)^2-(f_1(x,a)-r)^2\right]
= (f_i(x,a)-f_1(x,a))^2 \ge 0.
\]
Conditions~\cref{cond:informativeness} and \cref{cond:selfboundedness} are also satisfied with $\kappa_1=\sqrt{\frac{2K}{\gamma}}$ and $\kappa_2=4$, from prior work~\cite{Agarwal12Contextual}.  Plugging these values in \corref{cor:main} and choosing $\gamma=\Theta\left(\sqrt[3]{K/T}\right)$, we obtain the regret bound of $O\left(\sqrt{\ln\frac{1}{p_1}}K^{1/3}T^{2/3}\right)$, and the Bayes regret bound of $O\left(\sqrt{H(\vecb{p}^*+\kld{\vecb{p}^*}{\vecb{p}^*}}K^{1/3}T^{2/3}\right)$.

\subsection{Logarithmic Loss} \label{sec:logloss}


For logarithmic loss, we assume the shifted loss of all experts are bounded in $[-\beta/2,\beta/2]$ for some constant $\beta\in\Rset_+$, so that one can normalize the shifted logarithmic loss to the range of $[-1,1]$ by defining:
\begin{equation}
l_i(r|x,a) = \frac{\indfunc{r=1}}{\beta}\ln\frac{1}{f_i(x,a)} + \frac{\indfunc{r=0}}{\beta}\ln\frac{1}{1-f_i(x,a)} . \label{eqn:nll}
\end{equation}
This assumption can usually be satisfied in practice, and seems necessary to derive finite-time guarantees. 
Note that this assumption is slightly weaker than the more common assumption that the logarithmic loss itself is bounded (\eg, \cite{Filippi11Parametric}).

We now verify all necessary conditions.  \condref{cond:consistency} follows from the well-known fact that the expectation of logarithmic loss between the true expert and another is their KL-divergence,
\begin{eqnarray}
\E_{r|x,a}\left[\hat{l}_i(r|x,a)\right]
= \frac{1}{\beta}\kld{f_1(x,a)}{f_i(x,a)} \label{eqn:nll-mean}
\end{eqnarray}
which is in turn non-negative.

\condref{cond:informativeness} is verified in the following lemma:
\begin{lemma} \label{lem:logloss-i}
For the loss function defined in \eqnref{eqn:nll}, one has
\[
\Delta_i(x) \le K\sqrt{\frac{2\beta}{\gamma}}\sqrt{\E_{r,a|x}\left[\hat{l}_i(r|x,a)\right]} .
\]
\end{lemma}
\begin{proof}
We have the following:
\begin{eqnarray*}
\lefteqn{\Delta_i(x) = f_1(x,\expert_1(x)) - f_1(x,\expert_i(x))} \\
  &\le& \abs{f_1(x,\expert_1(x))-f_i(x,\expert_1(x))} + \abs{f_1(x,\expert_i(x))-f_i(x,\expert_i(x))} \\
  &\le& \sqrt{2\kld{f_1(x,\expert_1(x))}{f_i(x,\expert_1(x))}} + \sqrt{2\kld{f_1(x,\expert_i(x))}{f_i(x,\expert_i(x))}} \\
  &\le& \sum_a \sqrt{2\kld{f_1(x,a)}{f_i(x,a)}} \\
  &\le& \sqrt{2K} \sqrt{\sum_a \kld{f_1(x,a)}{f_i(x,a)}} \\
  &\le& \sqrt{\frac{2K}{\gamma/K}} \sqrt{\E_{a|x}\kld{f_1(x,a)}{f_i(x,a)}} \\
  &=& K\sqrt{\frac{2\beta}{\gamma}}\sqrt{\E_{r,a|x}\left[\hat{l}_i(r|x,a)\right]} ,
\end{eqnarray*}
where the first inequality is due to the triangle inequality; the second inequality is due to Pinsker's inequality; the fourth inequality is due to Jensen's inequality; the fifth inequality is from the fact that each arm is selected with probability at least $\gamma/K$; the last equality is from \eqnref{eqn:nll-mean}.
\end{proof}
\lihong{Improve dependence on $K$---apply $p\ge\gamma/K$ before Jensen's?}

\condref{cond:boundedness} is immediately satisfied by the normalization of $1/\beta$ in the definition of $l_i$ above.

\condref{cond:selfboundedness} is the most difficult one to verify.  To the best of our knowledge, such a result for logarithmic loss is not found in literature and can be of independent interest.  For example, it implies that the analysis of \citet{Agarwal12Contextual} for square loss also applies to the logarithmic loss.  The following lemma states the result more formally.  Its proof, which is rather technical, is left to the appendix.
\begin{lemma} \label{lem:logloss-sb}
For the loss function defined in \eqnref{eqn:nll}, there exists some constant $\kappa_2=O(1)$ such that \condref{cond:selfboundedness} holds.
\end{lemma}

With all four conditions verified, we can apply results in \secref{sec:main} to reach the regret bound of
$
O\left(K^{\frac{2}{3}}\beta^{\frac{1}{3}}T^{\frac{2}{3}}\sqrt{\ln\frac{1}{p_1}}\right) ,
$
and the Bayes regret bound of
$
O\left(K^{\frac{2}{3}}\beta^{\frac{1}{3}}T^{\frac{2}{3}}\sqrt{H(w^*)+\kld{w^*}{w}}\right) .
$

\section{Discussions} \label{sec:discussions}

In this paper, we propose a new family of algorithms, \gts, and analyze its regret in the expert-learning framework.  Our regret analysis provides a promising alternative to understanding the strong performance of \ts, an interesting and pressing research problem raised by its recent empirical success.  Compared to existing analysis in the literature, it has the following benefits.  First, the results apply more generally to a set of experts, rather than making specific modeling assumptions about the prior and likelihood.  Second, the analysis quantifies how the (not necessarily correct prior $\vecb{p}$) affects the regret bound, as well as the Bayes regret when optimal experts are drawn from an unknown prior $\vecb{p}^*$.  Similar to PAC-Bayes bounds, these results combine the benefits of good priors and the robustness of frequentist approaches.

Our proof for \gts\ is inspired by the online-learning literature~\cite{Cesabianchi06Prediction}.  However, a new technique is needed to prove the critical \lemref{lem:lbar}, which relies on self-boundedness of a loss function.  A similar property is shown by \cite{Agarwal12Contextual} for square loss only, and is used in a very different way.  The self-boundedness of logarithmic loss (\lemref{lem:logloss-sb}) appears new, to the best of our knowledge, and may be of independent interest.

\gts\ bears some similarities to the Regressor Elimination (RE) algorithm~\cite{Agarwal12Contextual}.  A crucial difference is that RE requires a computationally expensive operation of computing a ``balanced'' distribution over experts, in order to control variance in the elimination process.  In contrast, our algorithm is computationally much cheaper.  The operations of \gts\ are also related to EXP4~\cite{Auer02Nonstochastic}, which uses unbiased, importance-weighted reward estimates to do exponentiated updates of expert weights.  In practice, it seems more natural to use prediction loss of an expert to adjust its weight, rather than using the reward signals directly~\cite{Graepel10Web,Chapelle12Empirical}.

While we have focused on the case of finitely many experts, the setting is motivated by the more realistic case when the set $\expert$ of experts is continuous~\cite{Graepel10Web,Chapelle12Empirical,Agrawal13Thompson}.  The discrete case considered here may be thought of as an approximation to the continuous case, using a covering device.  We expect similar results to hold with $N$ replaced by the covering number of the class.


This work suggests a few interesting directions for future work.  The first is to close the gap between the current $O(T^{2/3})$ bound and the best problem-independent bound $O(\sqrt{T})$ for contextual bandits.  The second is to extend the analysis here to continuous expert classes, and more importantly to the agnostic (non-realizable) case.  Finally, it is interesting to use the regret analysis of (Generalized) \ts\ to obtain performance guarantees for its reinforcement-learning analogues (\eg, \cite{Strens00Bayesian}).





\appendix

\section{Proof of \lemref{lem:logloss-sb}: Self-boundedness of Logarithmic Loss}

This section proves \lemref{lem:logloss-sb}, regarding self-boundedness of logarithmic loss, in the sense described in \condref{cond:selfboundedness}.
The analysis here does not involve step $t$ and the corresponding context and selected arm.  We therefore simplify notation as follows: the true expert $\expert_1$ predicts $f\in(0,1)$, and the other expert $\expert_i$ predicts $g\in(0,1)$.  The binary reward is then a Bernoulli random variable with success rate $f$.  The shifted logarithmic loss of $\expert_2$ is given by
\[
\hat{l} = \frac{\indfunc{r=1}}{\beta}\ln\frac{f}{g} + \frac{\indfunc{r=0}}{\beta}\ln\frac{1-f}{1-g} .
\]
The first two moments of the random variable $\hat{l}$ are given by:
\begin{eqnarray*}
M_1 &\defeq& \E_r\left[\hat{l}\right] = \frac{\kld{f}{g}}{\beta}=\frac{f}{\beta}\ln\frac{f}{g}+\frac{1-f}{\beta}\ln\frac{1-f}{1-g} \\
M_2 &\defeq& \E_r\left[\hat{l}^2\right] = \frac{f}{\beta^2}\left(\ln\frac{f}{g}\right)^2 + \frac{1-f}{\beta^2}\left(\ln\frac{1-f}{1-g}\right)^2 .
\end{eqnarray*}
Define 
\[
F(g) \defeq \frac{M_2-M_1^2}{M_1} ,
\]
the ratio between variance and expectation of $\hat{l}$, as a function of $g$.  Our goal is to show that $F(g)$ is bounded by a constant, independent of $f$ and $g$.  It will then follow that $M_2/M_1$ is also bounded by a constant since
$M_2/M_1 = F(g) + M_1 \le F(g) + 1$.

Taking the derivative of $F$, one obtains
\[
F'(g) = -C(f,g) \ln\frac{f(1-g)}{g(1-f)}\left((f+g)\ln\frac{f}{g}+(2-f-g)\ln\frac{1-f}{1-g}\right)
\]
for some function $C(f,g)>0$.  It can be verified, by rather tedious calculations, that there exists some $g_0\in(0,1)$ such that $F'(g)\le0$ for $g<g_0$ and $F'(g)\ge0$ for $g>g_0$.  So $F(g)$ is maximized by making $g$ close to either $0$ or $1$.  It then follows, again by rather tedious calculations, that $F(g) = O(1)$, using the assumption that the log-ratios (that is, shifted loss) are bounded by $[-\beta/2,\beta/2]$.  \lihong{Double-check the last statement.}


\bibliographystyle{plain}
\bibliography{../../Reference/lihong_ml,../../Reference/lihong_rl,../../Reference/lihong_www}

\end{document}